\documentclass{article}

\usepackage{microtype}
\usepackage{graphicx}
\usepackage{subfigure}
\usepackage{booktabs} % for professional tables

\usepackage{amsmath,amsfonts,amsthm,bm}
\usepackage{bbm}

\newcommand{\comment}[1]{}

\newtheorem{theorem}{Theorem}
\newtheorem{lemma}{Lemma}

\newtheorem{assumption}{Assumption}

\def\eqref#1{equation~\ref{#1}}
\def\1{\mathbbm{1}}

\DeclareMathAlphabet{\mathsfit}{\encodingdefault}{\sfdefault}{m}{sl}
\SetMathAlphabet{\mathsfit}{bold}{\encodingdefault}{\sfdefault}{bx}{n}

\def\gD{{\mathcal{D}}}

\def\gG{{\mathcal{G}}}

\def\gP{{\mathcal{P}}}

\def\gX{{\mathcal{X}}}
\def\gY{{\mathcal{Y}}}

\newcommand{\E}{\mathbb{E}}

\newcommand{\R}{\mathbb{R}}

\DeclareMathOperator*{\argmin}{arg\,min}

\def\ytrue{y_{\mathrm{true}}}
\def\ybiased{y_{\mathrm{bias}}}
\def\ydummy{\hat{y}}

\def\X{\gX}
\def\Y{\gY}
\def\D{\gD}
\def\G{\gG}
\def\model{h}

\def\dkl{D_{\mathrm{KL}}}
\def\reals{\mathbb{R}}

\newtheorem{corollary}{Corollary}

\newtheorem{proposition}{Proposition}

\theoremstyle{definition}

\usepackage{hyperref}

\usepackage[accepted]{icml2018}

\icmltitlerunning{Identifying and Correcting Label Bias in Machine Learning}

\begin{document}

\twocolumn[
\icmltitle{Identifying and Correcting Label Bias in Machine Learning}

% It is OKAY to include author information, even for blind
% submissions: the style file will automatically remove it for you
% unless you've provided the [accepted] option to the icml2018
% package.

% List of affiliations: The first argument should be a (short)
% identifier you will use later to specify author affiliations
% Academic affiliations should list Department, University, City, Region, Country
% Industry affiliations should list Company, City, Region, Country

% You can specify symbols, otherwise they are numbered in order.
% Ideally, you should not use this facility. Affiliations will be numbered
% in order of appearance and this is the preferred way.
\icmlsetsymbol{equal}{*}

\begin{icmlauthorlist}
\icmlauthor{Heinrich Jiang}{equal,gr}
\icmlauthor{Ofir Nachum}{equal,gb}
\end{icmlauthorlist}

\icmlaffiliation{gr}{Google Research, Mountain View, CA}
\icmlaffiliation{gb}{Google Brain, Mountain View, CA}

\icmlcorrespondingauthor{Heinrich Jiang}{heinrich.jiang@gmail.com}
\icmlcorrespondingauthor{Ofir Nachum}{ofirnachum@google.com}

% You may provide any keywords that you
% find helpful for describing your paper; these are used to populate
% the "keywords" metadata in the PDF but will not be shown in the document
\icmlkeywords{Fairness}

\vskip 0.3in
]

% this must go after the closing bracket ] following \twocolumn[ ...

% This command actually creates the footnote in the first column
% listing the affiliations and the copyright notice.
% The command takes one argument, which is text to display at the start of the footnote.
% The \icmlEqualContribution command is standard text for equal contribution.
% Remove it (just {}) if you do not need this facility.

%\printAffiliationsAndNotice{}  % leave blank if no need to mention equal contribution
\printAffiliationsAndNotice{\icmlEqualContribution} % otherwise use the standard text.

\begin{abstract}
Datasets often contain biases which unfairly disadvantage certain groups, and classifiers trained on such datasets can inherit these biases. In this paper, we provide a mathematical formulation of how this bias can arise. We do so by assuming the existence of underlying, unknown, and unbiased labels which are overwritten by an agent who intends to provide accurate labels but may have biases against certain groups. Despite the fact that we only observe the biased labels, we are able to show that the bias may nevertheless be corrected by re-weighting the data points without changing the labels. We show, with theoretical guarantees, that training on the re-weighted dataset corresponds to training on the {\it unobserved} but {\it unbiased} labels, thus leading to an unbiased machine learning classifier. Our procedure is fast and robust and can be used with virtually any learning algorithm. We evaluate on a number of standard machine learning fairness datasets and a variety of fairness notions, finding that our method outperforms standard approaches in achieving fair classification.
\end{abstract}
\section{Introduction}

Machine learning has become widely adopted in a variety of real-world applications that significantly affect people's lives~\citep{tsa,loans}. Fairness in these algorithmic decision-making systems has thus become an increasingly important concern: It has been shown that without appropriate intervention during training or evaluation, models can be biased against certain groups~\citep{angwin,hardt2016equality}. 
This is due to the fact that the data used to train these models often contains biases that become reinforced into the model~\citep{bolukbasi2016man}. Moreover, it has been shown that simple remedies, such as ignoring the features corresponding to the protected groups, are largely
ineffective due to redundant encodings in the data \citep{pedreshi2008discrimination}.
In other words, the data can be inherently biased in possibly complex ways, thus making it difficult to achieve fairness.

Research on training fair classifiers has therefore received a great deal of attention. 
One such approach has focused on developing {\em post-processing} steps to enforce fairness on a learned model~\citep{doherty2012information,feldman2015computational,hardt2016equality}.  That is, one first trains a machine learning model, resulting in an unfair classifier.  The outputs of the classifier are then calibrated to enforce fairness.
Although this approach is likely to decrease the bias of the classifier, by decoupling the training from the fairness enforcement, this procedure may not lead to the best trade-off between fairness and accuracy.
Accordingly, recent work has proposed to incorporate fairness into the training algorithm itself, framing the problem as a constrained optimization problem and subsequently applying the method of {\em Lagrange multipliers} to transform the constraints to penalties~\citep{zafar2015fairness,goh2016satisfying,cotter2018two,agarwal2018reductions}; however such approaches may introduce undesired complexity and lead to more difficult or unstable training~\citep{cotter2018two,cotter2018optimization}.
Both of these existing methods address the problem of bias by adjusting the machine learning model rather than the data, despite the fact that oftentimes it is the training data itself -- i.e., the observed features and corresponding labels -- which are biased.

\begin{figure*}[t]
\centering
\includegraphics[width=2.09\columnwidth]{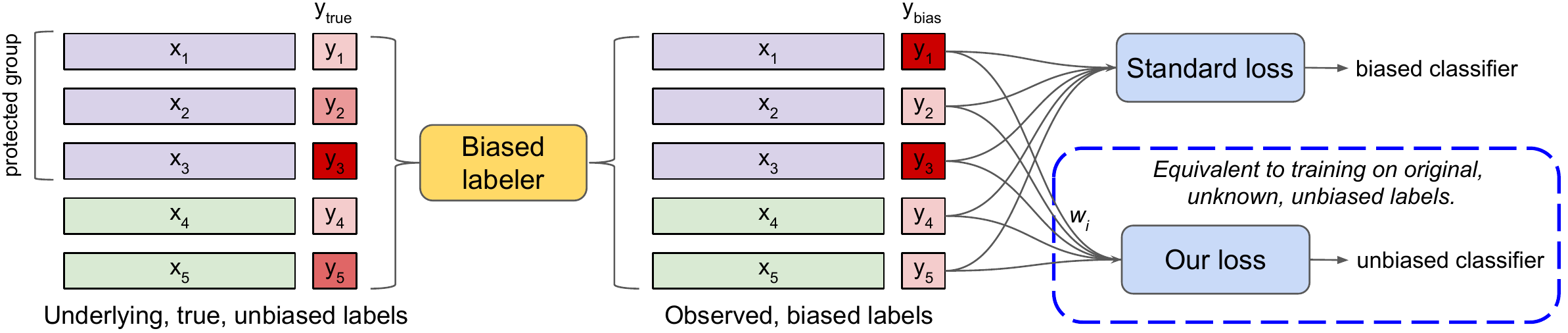}
\caption{In our approach to training an unbiased, fair classifier, we assume the existence of a true but unknown label function which has been adjusted by a biased process to produce the labels observed in the training data.  Our main contribution is providing a procedure that appropriately weights examples in the dataset, and then showing that training on the resulting loss corresponds to training on the original, true, unbiased labels.}
\label{fig:design}
\vspace{-0.2in}
\end{figure*}

In this paper, we provide an approach to machine learning fairness that addresses the underlying data bias problem directly.
We introduce a new mathematical framework for fairness in which we assume that there exists an {\it unknown} but {\it unbiased} ground truth label function and that the labels observed in the data are assigned by an agent who is possibly biased, but otherwise has the intention of being accurate. This assumption is natural in practice and may also be applied to settings where the features themselves are biased and that the observed labels were generated by a process depending on the features (i.e. situations where there is bias in both the features and labels). 

Based on this mathematical formulation, we show how one may identify the amount of bias in the training data as a closed form expression.  Furthermore, our derived form for the bias suggests that its correction may be performed by assigning appropriate weights to each example in the training data. We show, with theoretical guarantees, that training the classifier under the resulting weighted objective leads to an unbiased classifier on the original un-weighted dataset. Notably, many pre-processing approaches and even constrained optimization approaches (e.g. \citet{agarwal2018reductions}) optimize a loss which possibly modifies the observed labels or features, and doing so may be legally prohibited as it can be interpreted as training on {\it falsified} data; see~\citet{barocas2016big} (more details about this can be found in Section~\ref{section:related_work}). In contrast, our method {\em does not} 
modify any of the observed labels or features. Rather, we correct for the bias by changing the distribution of the sample points via re-weighting the dataset.

%Notably, the recent constrained optimization approach of \citet{agarwal2018reductions} to fair classification also leads to minimizing a re-weighted loss. However, in their case, depending on the fairness constraint and the Lagrange multipliers learned, the labels used in the re-weighted loss may not be the same as the original observed labels, thus penalizing the model for predicting correctly. While this may be a reasonable approach, training on {\it modified} labels may be legally prohibited \cite{barocas2016big}. Our framework leads to a re-weighted loss without any modifications of the data. 

Our resulting method is general and can be applied to various notions of fairness, including demographic parity, equal opportunity, equalized odds, and  disparate impact. Moreover, the method is practical and simple to tune: With the appropriate example weights, any off-the-shelf classification procedure can be used on the weighted dataset to learn a fair classifier. Experimentally, we show that on standard fairness benchmark datasets and under a variety of fairness notions our method can outperform previous approaches to fair classification.   
\section{Background}
In this section, we introduce our framework for machine learning fairness, which explicitly assumes an unknown and unbiased ground truth label function.
We additionally introduce notation and definitions used in the subsequent presentation of our method.

\subsection{Biased and Unbiased Labels}
Consider a data domain $\X$ and an associated data distribution $\gP$.
An element $x\in\X$ may be interpreted as a feature vector associated with a specific example.
We let $\Y := \{0, 1\}$ be the labels, considering the binary classification setting, although our method may be readily generalized to other settings. %Let us denote $\Delta(\Y) := [0, 1]$.
%Let $\Y$ be the space of possible labels.  For example, if our data is based on convict recidivism, then $\Y=\{0, 1\}$.  If our data is based on credit scores, then $\Y=[300, 850]$.

We assume the existence of an {\em unbiased}, ground truth label function $\ytrue:\X \to [0, 1]$.
%, where $\Delta(\Y)$ refers to the simplex on $\Y$; i.e.,
%\begin{equation}
%\Delta(\Y) = \{ f:\Y\to[0,\infty) ~\text{s.t.}~ \int_{\Y} f(y) \deriv y = %1 \}.
%\end{equation}
Although $\ytrue$ is the assumed ground truth, in general we do not have access to it.  Rather, our dataset has labels generated based on a {\em biased} label function $\ybiased:\X\to [0, 1]$.
%Accordingly, our dataset is presented either in {\em online} form, in which we receive successive samples of data points $(x, y)$ sampled according to
Accordingly, we assume that our data is drawn as follows:
\begin{equation*}
(x,y)\sim \D \equiv x\sim\gP, y\sim \text{Bernoulli}(\ybiased(x)).
\end{equation*}
and we assume access to a finite sample $\D_{[n]} :=\{(x^{(i)}, y^{(i)})\}_{i=1}^n$ drawn from $\D$.
%or in {\em offline} form, in which we have a fixed, finite collection of data points $\D=\{x^{(i)}, y^{(i)}\}_{i=1}^N$ sampled 
%from the same distribution $x^{(i)}\sim\gP, y^{(i)}\sim\ybiased(-|x^{(i)})$.
%In both cases, we will use $(x,y)\sim \D$ to denote a sample from the dataset, noting that in the offline case, successive samples from this distribution correspond to a single pass through the data.

In a machine learning context, our directive is to use the dataset $\D$ to recover the unbiased, true label function $\ytrue$.
In general, the relationship between the desired $\ytrue$ and the observed $\ybiased$ is unknown.  Without additional assumptions, it is difficult to learn a machine learning model to fit $\ytrue$.
We will attack this problem in the following sections by proposing a minimal assumption on the relationship between $\ytrue$ and $\ybiased$. The assumption will allow us to derive an expression for $\ytrue$ in terms of $\ybiased$, and the form of this expression will immediately imply that correction of the label bias may be done by appropriately re-weighting the data.
We note that our proposed perspective on the problem of learning a fair machine learning model is conceptually different from previous ones.
While previous perspectives propose to train on the observed, biased labels and only enforce fairness as a penalty or as a post-processing step to the learning process, we take a more direct approach. Training on biased data can be inherently misguided, and thus we believe that our proposed perspective may be more appropriate and better aligned with the directives associated with machine learning fairness.

\subsection{Notions of Bias}
We now discuss precise ways in which $\ybiased$ can be biased. We describe a number of accepted notions of fairness; i.e., what it means for an arbitrary label function or machine learning model $\model:\X\to[0, 1]$ to be biased (unfair) or unbiased (fair). 

We will define the notions of fairness in terms of a {\em constraint function} $c:\X \times \Y \to \R$. %, where $\R(\Y)=\{f:\Y\to\R\}$ is used to denote the set of real-valued functions on $\Y$.
%We define the inner product, 
%\begin{equation}
%    \langle \model(x), c(x) \rangle = \int_{\Y} \model(y|x) c(x, y) \deriv y.
%\end{equation}
%where we use the shorthand $c(x,y) = c(x)(y)$. 
%Note $\langle \model(x), c(x) \rangle = \E_{y\sim\model(-|x)}[c(x, y)]$.
Many of the common notions of fairness may be expressed or approximated as linear constraints on $\model$ (introduced previously by \citet{cotter2018optimization,goh2016satisfying}).  That is, they are of the form
\begin{equation*}
    \E_{x \sim \gP}\left[ \langle \model(x), c(x) \rangle \right] = 0,
\end{equation*}
where $\langle \model(x), c(x) \rangle := \sum_{y \in \mathcal{Y}} \model(y|x) c(x, y)$ and we use the shorthand $h(y|x)$ to denote the probability of sampling $y$ from a Bernoulli random variable with $p=h(x)$; i.e., $h(1|x) := h(x)$ and $h(0|x) := 1 - h(x)$.
%\begin{equation*}
%    \E_{(x, y) \sim \D}\left[ \model(x)\cdot c(x, y) \right] = 0.
%\end{equation*}
%for some $r \in \R$.
Therefore, a label function $\model$ is {\it unbiased} with respect to the constraint function $c$ if $\E_{x \sim \gP}\left[ \langle \model(x), c(x) \rangle \right] = 0$.  If $\model$ is biased, the degree of bias (positive or negative) is given by $\E_{x \sim \gP}\left[ \langle \model(x), c(x) \rangle \right]$.

We define the notions of fairness with respect to a protected group $\G$, and thus assume access to an indicator function $g(x) = \1[x\in \G]$.
We use $Z_\G:=\E_{x\sim \gP}[g(x)]$ to denote the probability of a sample drawn from $\gP$ to be in $\G$.
We use
$P_\X=\E_{x \sim \gP}[\ytrue(x)]$ to denote the proportion of $\X$ which is positively labelled and
$P_\G=\E_{x\sim \gP}[g(x)\cdot \ytrue(x)]$ to denote the proportion of $\X$ which is positively labelled and in $\G$. %{\bf [TODO: Should this use ybiased??]}.
We now give some concrete examples of accepted notions of constraint functions:\\
{\bf Demographic parity} \citep{dwork2012fairness}: A fair classifier $\model$ should make positive predictions on $\G$ at the same rate as on all of $\X$. The constraint function may be expressed as $c(x, 0) = 0$, $c(x, 1) = g(x)/Z_\G - 1$.\\
 {\bf Disparate impact} \citep{feldman2015certifying}: %A fair classifier $\model$ should have a rate of positive prediction on $\G$ at least $p\%$ as high as the rate of positive prediction on all of $\X$, and vice-versa.  Traditionally, $p=80$.  Unlike other notions of fairness, disparate impact {\em may not} be expressed as a linear constraint.  
    %Nevertheless, previous work~\citep{zafar2015fairness} has suggested approximating it as such; i.e. $\Y=\{0,1\}$ and the constraint is expressed as $\E_{x\sim\gP}[\langle h(x), c(x)\rangle] \in [-\epsilon, \epsilon]$, where $c(x, 0) = 0$ and $c(x, 1) = g(x) / Z_\G - 1$.
This is identical to demographic parity, only that, in addition, during inference the classifier does not have access to the features of $x$ indicating whether it belongs to the protected group.\\
{\bf Equal opportunity} \citep{hardt2016equality}: A fair classifier $\model$ should have equal true positive rates on $\G$ as on all of $\X$. The constraint may be expressed as $c(x,0) = 0$, $c(x,1) = g(x)\ytrue(x) / P_\G - \ytrue(x)/P_\X$.\\
{\bf Equalized odds} \citep{hardt2016equality}: A fair classifier $\model$ should have equal true positive and false positive rates on $\G$ as on all of $\X$.
    In addition to the constraint associated with equal opportunity, this notion applies an additional constraint with $c(x, 0) = 0$, $c(x, 1) = g(x)(1 - \ytrue(x))/(Z_\G - P_\G) - (1 - \ytrue(x)) / (1 - P_\X)$.
    
In practice, there are often multiple fairness constraints $\{c_k\}_{k=1}^K$ associated with multiple protected groups $\{\G_k\}_{k=1}^K$. It is clear that our subsequent results will assume multiple fairness constraints and protected groups, and that the protected  groups may have overlapping samples.

%\section{Deriving Unbiased Labels from Biased Observations}
%\subsection{Modeling how bias arises in data}
% [Ofir] I think this should be a new section, since it is part of our method; i.e., not something appearing in previous work.
\vspace{-0.1in}
\section{Modeling How Bias Arises in Data}
We now introduce our underlying mathematical framework to understand bias in the data, by providing the relationship between $\ybiased$ and $\ytrue$ (Assumption~\ref{ass:bias} and Proposition~\ref{prop:closedform}).
This will allow us to derive a closed form expression for $\ytrue$ in terms of $\ybiased$ (Corollary~\ref{cor:closed}). In Section~\ref{sec:learning} we will show how this expression leads to a simple weighting procedure that uses data with biased labels to train a classifier with respect to the true, unbiased labels.

%We have access to a dataset $\D_{[n]}$ labelled according to a biased label function $\ybiased$.  Nevertheless, we
We begin with an assumption on the relationship between the observed $\ybiased$ and the underlying $\ytrue$.
%In general, without any such assumption, there is no way to know what the true label function is or to fit a machine learning classifier to it.
%One must assume some relationship between the unknown $\ytrue$ and the observed $\ybiased$.
%To this end, we propose the following assumption:
\begin{assumption}
\label{ass:bias}
Suppose that our fairness constraints are $c_1,..,.c_K$, with respect to which $\ytrue$ is unbiased (i.e. $\E_{x \sim \gP}\left[ \langle \ytrue(x), c_k(x) \rangle  \right] = 0$ for $k \in [K]$). We assume that there 
exist $\epsilon_1,\dots,\epsilon_K\in\R$ such that the observed, biased label function $\ybiased$ is the solution of the following constrained optimization problem:
\begin{align*}
\argmin_{\ydummy:\X\to[0, 1]} & \E_{x\sim \gP}\left[\dkl(\ydummy(x) || \ytrue(x)) \right] \\
\mathrm{s.t.}~~ & \E_{x\sim \gP}\left[ \langle \ydummy(x), c_k(x) \rangle  \right] =\epsilon_k \\
& \mathrm{for}~ k=1,\dots,K,
\end{align*}
where we use $\dkl$ to denote the KL-divergence.
\end{assumption}
In other words, we assume that $\ybiased$ is the label function closest to $\ytrue$ while achieving some amount of bias, where proximity to $\ytrue$ is given by the KL-divergence.  
%That is, our assumption states that $\ybiased$ is constructed in such a way that it matches the true labels $\ytrue$ as much as possible, while being biased to a desired degree.
%
This is a reasonable assumption in practice, where the observed data may be the result of manual labelling done by actors (e.g. human decision-makers) who strive to provide an accurate label while being affected by (potentially unconscious) biases; or in cases where the observed labels correspond to a process (e.g. results of a written exam) devised to be accurate and fair, but which is nevertheless affected by inherent biases.

We use the KL-divergence to impose this desire to have an accurate labelling.  In general, a different divergence may be chosen.  However in our case, the choice of a KL-divergence allows us to derive the following proposition, which provides a closed-form expression for the observed $\ybiased$. 
The derivation of Proposition~\ref{prop:closedform} from Assumption~\ref{ass:bias} is standard and has appeared in previous works; e.g.  \citet{friedlander2006minimizing,botev2011generalized}. For completeness, we include the proof in the Appendix.
\begin{proposition}\label{prop:closedform}
Suppose that Assumption~\ref{ass:bias} holds.
Then $\ybiased$ satisfies the following for all $x \in \X$ and $y \in \Y$.
\vspace{-0.04in}
\begin{equation*}
    \ybiased(y|x) \propto \ytrue(y|x) \cdot \exp \left\{-\sum_{k=1}^K \lambda_k\cdot c_k(x,y) \right\}
\end{equation*}
\vspace{-0.06in}
for some $\lambda_1,\dots,\lambda_K\in \R$. % where we let $\ybiased(y|x)$  denote $\ybiased(x)$ if $y=1$ and $1 - \ybiased(x)$ if $y=0$.  $\ytrue(y|x)$ is defined similarly.
\end{proposition}
Given this form of $\ybiased$ in terms of $\ytrue$, we can immediately deduce the form of $\ytrue$ in terms of $\ybiased$:
\begin{corollary}
\label{cor:closed}
Suppose that Assumption~\ref{ass:bias} holds.
The unbiased label function $\ytrue$ is of the form,
\vspace{-0.04in}
\begin{equation*}
    \ytrue(y|x) \propto \ybiased(y|x) \cdot \exp\left\{ \sum_{k=1}^K \lambda_k c_k(x,y) \right\},
\end{equation*}
\vspace{-0.06in}
for some $\lambda_1,\dots,\lambda_K\in \R$.
\end{corollary}

We note that previous approaches to learning fair classifiers often formulate a constrained optimization problem similar to that appearing in Assumption~\ref{ass:bias} (i.e., maximize the accuracy or log-likelihood of a classifier subject to linear constraints) and subsequently solve it, usually via the method of Lagrange multipliers which translates the constraints to penalties on the training loss.  In our approach, rather than using the constrained optimization problem to formulate a machine learning objective, we use it to express the relationship between true (unbiased) and observed (biased) labels.  Furthermore, rather than training with respect to the biased labels, our approach aims to recover the true underlying labels.  As we will show in the following sections, this may be done by simply optimizing the training loss on a {\em re-weighting} of the dataset.
In contrast, the penalties associated with Lagrangian approaches can often be cumbersome: The original, non-differentiable, fairness constraints must be relaxed or approximated before conversion to penalties.  Even then, the derivatives of these approximations may be near-zero for large regions of the domain, causing difficulties during training.

\section{Learning Unbiased Labels}
\label{sec:learning}
We have derived a closed form expression for the true, unbiased label function $\ytrue$ in terms of the observed label function $\ybiased$, coefficients $\lambda_1,\dots,\lambda_K$, and constraint functions $c_1,\dots,c_K$.
In this section, we elaborate on how one may learn a machine learning model $\model$ to fit $\ytrue$, given access to a dataset $\D$ with labels sampled according to $\ybiased$.
We begin by restricting ourselves to constraints $c_1,\dots,c_K$ associated with demographic parity, allowing us to have full knowledge of these constraint functions.  In Section~\ref{sec:extension} we will show how the same method may be extended to general notions of fairness.

With knowledge of the functions $c_1,\dots,c_K$, it remains to determine the coefficients $\lambda_1,\dots,\lambda_K$ (which give us a closed form expression for the dataset weights) as well as the classifier $h$.
For simplicity, we present our method by first showing how a classifier $\model$ may be learned assuming knowledge of the coefficients $\lambda_1,\dots,\lambda_K$ (Section~\ref{subsection:learn_h_given_lambda}). We subsequently show how the coefficients themselves may be learned, thus allowing our algorithm to be used in general setting (Section~\ref{subsection:learn_lambda}). Finally, we describe how to extend to more general notions of fairness (Section~\ref{sec:extension}).
% NEW

\subsection{Learning $\model$ Given $\lambda_1,\dots,\lambda_K$}\label{subsection:learn_h_given_lambda}
Although we have the closed form expression $\ytrue(y|x) \propto \ybiased(y|x) \exp\left\{ \sum_{k=1}^K \lambda_k c_k \right\}$ for the true label function, in practice we do not have access to the values $\ybiased(y|x)$ but rather only access to data points with labels sampled from $\ybiased(y|x)$.  We propose 
the {\em weighting} technique to train $\model$ on labels based on $\ytrue$.\footnote{See the Appendix for an alternative to the weighting technique -- the {\em sampling} technique, based on a coin-flip.}  
%{\bf Weighting}: 
The weighting technique weights an example $(x,y)$ by the weight
$w(x,y)=\widetilde{w}(x,y) / \sum_{y^\prime\in\Y}\widetilde{w}(x,y^\prime)$,
%\begin{align*}
    %w(x, y) = \frac{\exp\left\{ \sum_{k=1}^K \lambda_k c_k(x,y) \right\}}{\sum_{y^\prime\in\Y} \exp\left\{ \sum_{k=1}^K \lambda_k c_k(x,y^\prime) \right\}}.
%\end{align*}
where 
\begin{align*}
    \widetilde{w}(x,y^\prime) = \exp\bigg\{ \sum_{k=1}^K \lambda_k c_k(x,y^\prime) \bigg\}.
\end{align*}
%$w(x,y)=\widetilde{w}(x,y) / \sum_{y^\prime\in\Y} \widetilde{w}(x,y^\prime)$, where $\widetilde{w}(x,y^\prime)=\exp\left\{ \sum_{k=1}^K \lambda_k c_k(x,y^\prime) \right\}$.  
We have the following theorem, which states that training a classifier on examples with biased labels weighted by $w(x,y)$ is equivalent to training a classifier on examples labelled according to the true, unbiased labels.  
\begin{theorem}
\label{thm:equivalent}
For any loss function $\ell$, training a classifier $\model$ on the weighted objective $\E_{x\sim\gP,y\sim\ybiased(x)}[w(x,y) \cdot \ell(h(x), y)]$ is equivalent to training the classifier on the objective $\E_{x\sim\tilde{\gP},y\sim\ytrue(x)}[\ell(h(x), y)]$ with respect to the underlying, true labels, for some distribution $\tilde{\gP}$ over $\X$.
\end{theorem}
\begin{proof}
For a given $x$ and for any $y\in\Y$, due to Corollary~\ref{cor:closed} we have,
\begin{equation}
    w(x,y) \ybiased(y|x) = \phi(x) \ytrue(y|x),
\end{equation}
where $\phi(x) = \sum_{y^\prime \in \Y} w(x,y^\prime) \ybiased(y^\prime|x)$ depends only on $x$.
Therefore, letting $\tilde{\gP}$ denote the feature distribution $\tilde{\gP}(x)\propto\phi(x)\gP(x)$, we have,
\begin{multline}
    \E_{x\sim\gP,y\sim\ybiased(x)}[w(x,y) \cdot \ell(h(x), y)] = \\ C\cdot\E_{x\sim\tilde{\gP},y\sim\ytrue(x)}[\ell(h(x), y)], 
\end{multline}
where $C=\E_{x\sim\gP}[\phi(x)]$, and this completes the proof.
\end{proof}
Theorem~\ref{thm:equivalent} is a core contribution of our work.  It states that the bias in observed labels may be corrected in a simple and straightforward way: Just re-weight the training examples.  We note that Theorem~\ref{thm:equivalent} suggests that when we re-weight the training examples, we trade off the ability to train on unbiased labels for training on a slightly different distribution $\tilde{\gP}$ over features $x$.
In Section~\ref{sec:theory}, we will show that given some mild conditions, the change in feature distribution does not affect the bias of the final learned classifier.  Therefore, in these cases, training with respect to weighted examples with biased labels is equivalent to training with respect to the same examples and the true labels.

%{\bf Sampling}: The sampling technique is based on the observation that closed-form distribution leads to a natural way of sampling from the training data and ignoring or skipping some training data based on a coin-flip. Due to space limitations and the fact that we use the weighting technique for the theoretical and experimental analysis of this paper, we defer the sampling technique to the Appendix. 

\subsection{Determining the Coefficients $\lambda_1,\dots,\lambda_K$}\label{subsection:learn_lambda}
We now continue to describe how to learn the coefficients $\lambda_1,\dots,\lambda_K$.
One advantage of our approach is that, in practice, $K$ is often small.  Thus, we propose to iteratively learn the coefficients so that the final classifier satisfies the desired fairness constraints either on the training data or on a validation set. We first discuss how to do this for demographic parity and will discuss extensions to other notions of fairness in Section~\ref{sec:extension}. See the full pseudocode for learning $\model$ and $\lambda_1,\dots,\lambda_K$ in Algorithm~\ref{alg:demparity}.

Intuitively, the idea is that if the positive prediction rate for a protected class $\G$ is lower  than the overall positive prediction rate, then the corresponding coefficient should be increased; i.e., if we increase the weights of the positively labeled examples of $\G$ and  decrease the weights of the negatively labeled examples of $\G$, then this will encourage the classifier to increase its accuracy on the positively labeled examples in $\G$, while the accuracy on the negatively labeled examples of $\G$ may fall. Either of these two events will cause the positive prediction rate on $\G$ to increase, and thus bring $\model$ closer to the true, unbiased label function.

Accordingly, Algorithm~\ref{alg:demparity} works by iteratively performing the following steps: (1) evaluate the demographic parity constraints; (2) update the coefficients by subtracting the respective constraint violation multiplied by a fixed step-size; (3) compute the weights for each sample based on these multipliers using the closed-form provided by Proposition~\ref{prop:closedform}; and (4) retrain the classifier given these weights.  

Algorithm~\ref{alg:demparity} takes in a classification procedure $H$, which given a dataset $D_{[n]} := \{(x_i,y_i)\}_{i=1}^n$ and weights $\{w_i\}_{i=1}^n$, outputs a classifier. In practice, $H$ can be any training procedure which minimizes a weighted loss function over some parametric function class (e.g. logistic regression).

Our resulting algorithm simultaneously minimizes the weighted loss and maximizes fairness via learning the coefficients, which may be interpreted as competing goals with different objective functions. Thus, it is a form of a non-zero-sum two-player game. The use of non-zero-sum two-player games in fairness was first proposed in~\citet{cotter2018two} for the Lagrangian approach.
%, which modifies the Lagrangian for the fairness objective.

%in our theoretical result later, we assume that $H$ has the following form:
%\begin{align*}
%H(D_{[n]}, \{w_i\}_{i=1}^n) :=  \argmin_{h \in \mathcal{H}} \sum_{i=1}^n w_i \cdot (h(x_i) -  y_i)^2,
%\end{align*}
%where $\mathcal{H}$  is a class of smooth functions from $\X$ to $[0, 1]$. In practice, we can let $\mathcal{H}$ be some parametric function class (i.e. logistic regression or neural network) and $H$ is a training procedure which tries to minimize the weighted loss function.

\begin{algorithm}[t]
   \caption{Training a fair classifier for Demographic Parity, Disparate Impact, or Equal Opportunity.}
   \label{alg:demparity}
\begin{algorithmic}
   \STATE {\bf Inputs}: Learning rate $\eta$, number of loops $T$, training data $\D_{[n]} = \{(x_i, y_i)\}_{i=1}^N$, classification procedure $H$. constraints $c_1,...,c_K$ corresponding to protected groups $\G_1,...,\G_K$.
   \STATE Initialize $\lambda_1,...,\lambda_K$ to $0$ and $w_1 = w_2 = \cdots = w_n = 1$.
   \STATE Let $h := H(\D_{[n]}, \{w_i\}_{i=1}^n)$
   \FOR{$t = 1,...,T$}
   \STATE Let $\Delta_k := \E_{x\sim\D_{[n]}}[\langle h(x), c_k(x) \rangle]$ for $k \in [K]$.
   \STATE Update $\lambda_k = \lambda_k - \eta \cdot \Delta_k$ for $k  \in [K]$.
   \STATE Let $\widetilde{w_i} := \exp\left( \sum_{k=1}^K \lambda_k \cdot \1[x \in \G_k] \right)$ for $i \in [n]$
   \STATE Let $w_i = \widetilde{w_i} / (1 + \widetilde{w_i})$ if $y_i = 1$, otherwise $w_i = 1 / (1 + \widetilde{w_i})$ for $i \in [n]$
   \STATE Update $h = H(\D_{[n]}, \{w_i\}_{i=1}^n)$
   \ENDFOR 
   \STATE {\bf Return} $h$
\end{algorithmic}
\end{algorithm}

\subsection{Extension to Other Notions of Fairness}
\label{sec:extension}
The initial restriction to demographic parity was made so that the values of the constraint functions $c_1,\dots,c_K$ on any $x\in\X,y\in\Y$ would be known. We note that Algorithm~\ref{alg:demparity} works for disparate impact as well: The only change would be that the classifier does not have access to the protected attributes.
However, in other notions of fairness, such as equal opportunity or equalized odds, the constraint functions depend on $\ytrue$, which is unknown.  %For example, equal opportunity uses the constraint function,
%\begin{equation*}
%    c(x,y) = \ytrue(x)\left(\frac{g(x)}{P_\G} - \frac{1}{P_\X}\right)  
%    ~\text{if}~ y = 1 ~\text{else}~ 0.
%\end{equation*}

For these cases, we propose to apply the same technique of iteratively re-weighting the loss to achieve the desired fairness notion, with the weights $w(x,y)$ on each example determined only by the protected attribute $g(x)$ and the observed label $y\in\Y$.
This is equivalent to using Theorem~\ref{thm:equivalent} to derive the same procedure presented in Algorithm~\ref{alg:demparity}, but approximating the unknown constraint function $c(x,y)$ as a piece-wise constant function $d(g(x),y)$, where $d:\{0,1\}\times\Y\to\reals$ is unknown.
Although we do not have access to $d$, we may treat $d(g(x),y)$ as an additional set of parameters -- one for each protected group attribute $g(x)\in\{0,1\}$ and each label $y\in\Y$.  
These additional parameters may be learned in the same way the $\lambda$ coefficients are learned.  In some cases, their values may be wrapped into the unknown coefficients.
For example, for equal opportunity, there is in fact no need for any additional parameters. % (see elaboration below).  
On the other hand, for equalized odds, the unknown values for $\lambda_1,\dots,\lambda_K$ and $d_1,\dots,d_K$, are instead treated as unknown values for $\lambda_1^{TP},\dots,\lambda_K^{TP},\lambda_1^{FP},\dots,\lambda_K^{FP}$; i.e., separate coefficients for positively and negatively labelled points.
Due to space constraints, see the Appendix for further details on these and more general constraints.
%We further note that in practice, for fairness metrics that require the labels (such as equal opportunity and equalized odds), the goal is often to show that these fairness constraints hold relative to the {\it observed} labels, rather than the unobserved ground truth. We elaborate on the extension of our algorithm to Equal Opportunity below and defer the extensions for Equalized Odds and more general constraints to the Appendix. 

\vspace{-0.1in}
\section{Theoretical Analysis}
\label{sec:theory}
In this section, we provide theoretical guarantees on a learned classifier $\model$ using the weighting technique. We show that with the coefficients $\lambda_1,...,\lambda_K$ that satisfy Proposition~\ref{prop:closedform}, training on the re-weighted dataset leads to a finite-sample non-parametric rates of consistency on the estimation error provided the classifier has sufficient flexibility. 

%The goal is to show that for demographic parity, with the Lagrange multipliers that satisfy Proposition~\ref{prop:closedform}, training on the re-weighted dataset leads to a finite-sample non-parametric bound on the bias if the classifier has sufficient flexibility. 

We need to make the following regularity assumption on the data distribution, which assumes that the data is supported on a compact set in $\mathbb{R}^D$ and $\ybiased$ is smooth (i.e. Lipschitz).

\begin{assumption}\label{assumption:smoothness}
$\X$ is a compact set over $\mathbb{R}^D$ and both $\ybiased(x)$ and  $\ytrue(x)$ are $L$-Lipschitz (i.e. $|\ybiased(x) - \ybiased(x')| \le L\cdot |x - x'|$).
\end{assumption}

 We now give the result. The proof is technically involved and is deferred to the Appendix due to space.
 
\begin{theorem}[Rates of Consistency]\label{theo:rates}
Let $0 < \delta < 1$. Let $\D_{[n]} = \{(x_i,y_i)\}_{i=1}^n$ be a sample drawn from $\D$. Suppose that Assumptions~\ref{ass:bias} and~\ref{assumption:smoothness} hold. 
Let $\mathcal{H}$ be the set of all $2L$-Lipschitz functions mapping $\mathcal{X}$ to $[0, 1]$. 
 Suppose that the constraints are $c_1,...,c_K$ and the corresponding coefficients $\lambda_1,...,\lambda_K$ satisfy Proposition~\ref{prop:closedform} where $-\Lambda  \le \lambda_k \le \Lambda$ for $k =1,...,K$ and some $\Lambda > 0$.
 Let $h^*$ be the optimal function in $\mathcal{H}$ under the weighted mean square error objective,
 where the weights satisfy Proposition~\ref{prop:closedform}.  
 Then there exists $C_0$ depending on $\mathcal{D}$ such that for $n$ sufficiently large depending on $\mathcal{D}$, we have with probability at least $1 - \delta$:
 \begin{align*}
   ||h^* - \ytrue||_2 \le  C_0\cdot \log(2/\delta)^{1/(2+D)}\cdot  n^{-1/(4+2D)}.
\end{align*}
where $||h - h'||_{2} :=  \E_{x \sim \mathcal{P}}[(h(x) - h'(x))^2]$.
\end{theorem}

Thus, with the appropriate values of $\lambda_1$,...,$\lambda_K$ given by Proposition~\ref{prop:closedform}, we see that training with the weighted dataset based on these values will guarantee that the final classifier will be close to $\ytrue$. However, the above rate has a dependence on the dimension $D$, which may be unattractive in high-dimensional settings. If the data lies on a $d$-dimensional submanifold, then Theorem~\ref{theo:rates_manifold} below says that without any changes to the procedure, we will enjoy a rate that depends on the manifold dimension and independent of the ambient dimension. Interestingly, these rates are attained without knowledge of the manifold or its dimension.

\begin{theorem}[Rates on Manifolds]\label{theo:rates_manifold}
Suppose that all of the conditions of Theorem~\ref{theo:rates} hold and that in addition, $\mathcal{X}$ is a $d$-dimensional Riemannian submanifold of $\mathbb{R}^D$ with finite volume and finite condition number.
Then there exists $C_0$ depending on $\mathcal{D}$ such that for $n$ sufficiently large depending on $\mathcal{D}$, we have with probability at least $1 - \delta$:
 \begin{align*}
   ||h^* - \ytrue||_{2} \le  C_0\cdot \log(2/\delta)^{1/(2+d)}\cdot  n^{-1/(4+2d)}.
\end{align*}

\end{theorem}
\section{Related Work}\label{section:related_work}

Work in fair classification can be categorized into three approaches: post-processing of the outputs, the Lagrangian approach of transforming constraints to penalties, and pre-processing training data. 

{\bf Post-processing}: One approach to fairness is to perform a post-processing of the classifier outputs. Examples of previous work in this direction include \citet{doherty2012information,feldman2015computational,hardt2016equality}. However, this approach of calibrating the outputs to encourage fairness has limited flexibility. \citet{pleiss2017fairness} showed that a deterministic solution is only compatible with a single error constraint and thus cannot be applied to fairness notions such as equalized odds. Moreover, decoupling the training and calibration can lead to models with poor accuracy trade-off. In fact  \citet{woodworth2017learning} showed that in certain cases, post-processing can be provably suboptimal. Other works discussing the incompatibility of fairness notions include \citet{chouldechova2017fair,kleinberg2016inherent}. 

{\bf Lagrangian Approach}: There has been much recent work done on enforcing fairness by transforming the constrained optimization problem via the method of Lagrange multipliers. Some works \cite{zafar2015fairness,goh2016satisfying} apply this to the convex setting.
In the non-convex case, there is work which frames the constrained optimization problem as a two-player game~\citep{kearns2017preventing,agarwal2018reductions,cotter2018two} .
%The resulting classifier in this case is
%a randomized distribution over classifiers.
Related approaches include \citet{edwards2015censoring,corbett2017algorithmic,narasimhan2018learning}.
There is also recent work similar in spirit which encourages fairness by adding penalties to the objective; e.g, \citet{donini2018empirical} studies this for kernel methods and \citet{komiyama2018nonconvex} for linear models.
%However, there are a number of underlying challenges to the constrained optimization approach. For one, 
However, the fairness constraints are often irregular and have to be relaxed in order to optimize.
Notably, our method does not use the constraints directly in the model loss, and thus does not require them to be relaxed. %\cite{cotter2018optimization} has tried to address this issue, which leads to a complicated swap-regret strategy solving a non-zero-sum game.
Moreover, these approaches typically are not readily applicable to equality constraints as feasibility challenges can arise; thus, there is the added challenge of determining appropriate slack during training. Finally, the training can be difficult as \citet{cotter2018optimization} has shown that the Lagrangian may not even have a solution to converge to. 

When the classification loss and the relaxed constraints have the same form (e.g. a hinge loss as in~\citet{eban2017scalable}),
the resulting Lagrangian may be rewritten as a cost-sensitive classification, explicitly pointed out in \citet{agarwal2018reductions}, who show that the Lagrangian method reduces to solving an objective of the form $\sum_{i=1}^n w_i \1 [h(x_i) \neq y_i']$ for some non-negative weights $w_i$. In this setting, $y_i'$ may not necessarily be the true label, which may occur for example in demographic parity when the goal is to predict more positively within a protected group and thus may be penalized for predicting correctly on negative examples. While this may be a reasonable approach to achieving fairness, it could be interpreted as training a weighted loss on {\it modified} labels, which may be legally prohibited \cite{barocas2016big}.
Our approach is a non-negative re-weighting of the original loss (i.e., does not modify the observed labels) and is thus simpler and more aligned with legal standards. %Thus, in the context of fairness, our approach can be viewed as a simplified way to solve what the Lagrangian methods aims to achieve, but in a more direct fashion.

{\bf Pre-processing:} This approach has primarily involved massaging the data to remove bias. Examples include \citet{calders2009building,kamiran2009classifying,vzliobaite2011handling,kamiran2012data,zemel2013learning,fish2015fair,feldman2015certifying,beutel2017data}. Many of these approaches involve changing the labels and features of the training set, which may have legal implications since it is a form of training on falsified data \cite{barocas2016big}. Moreover, these approaches typically do not perform as well as the state-of-art and have thus far come with few theoretical guarantees \cite{krasanakis2018adaptive}.
In contrast, our approach does not modify the training data and only re-weights the importance of certain sensitive groups. Our approach is also notably based on a mathematically grounded formulation of how the bias arises in the data. 

\section{Experiments}

\begin{table*}
\setlength{\tabcolsep}{5pt}
\begin{center}
  \caption{{\bf Experiment Results: Benchmark Fairness Tasks}: Each row corresponds to a dataset and fairness notion.  We show the accuracy and fairness violation of training with no constraints (Unc.), with  post-processing calibration (Cal.), the Lagrangian approach (Lagr.) and our method. {\bf Bolded} is the method achieving lowest fairness violation for each row. All reported numbers are evaluated on the test set.}
  \label{tab:experiments}
  \begin{tabular}{|l|l|rr|rr|rr|rr|}
 \hline
Dataset & Metric & Unc. Err. & Unc. Vio. & Cal. Err. & Cal. Vio. & Lagr. Err. & Lagr. Vio. & Our Err. &  Our Vio. \\
\hline
Bank & Dem. Par. & 9.41\%  & .0349 & 9.70\% & .0068 & 10.46\% & .0126 &  9.63\% & {\bf  .0056} \\
 & Eq. Opp. & 9.41\%  & .1452 & 9.55\% & .0506 & 9.86\% & .1237 &  9.48\% & {\bf .0431} \\
 & Eq. Odds & 9.41\%  & .1452 & N/A & N/A  & 9.61\% & .0879 &  9.50\% & {\bf .0376} \\
 & Disp. Imp. & 9.41\%  & .0304 & N/A & N/A  & 10.44\% & .0135 &  9.89\% & {\bf .0063} \\
\hline
COMPAS  & Dem. Par. & 31.49\%  & .2045 & 32.53\% & .0201 & 40.16\% & .0495 & 35.44\% & {\bf .0155} \\
  &Eq. Opp. & 31.49\%  & .2373 & 31.63\% & {\bf .0256} & 36.92\% & .1141 & 33.63\% & .0774 \\
  & Eq. Odds & 31.49\%  & .2373 &  N/A & N/A  & 42.69\% & {\bf .0566} & 35.06\% & .0663 \\
  & Disp. Imp. & 31.21\%  & .1362 &  N/A & N/A  & 40.35\% & .0499 & 42.64\% & {\bf .0256} \\
\hline
Communities & Dem. Par.& 11.62\% & .4211 & 32.06\%  & .0653 &   28.46\% & .0519 & 30.06\% & {\bf .0107} \\
  & Eq. Opp.& 11.62\% & .5513 & 17.64\%  & {\bf .0584} & 28.45\% & .0897 & 26.85\% & .0833 \\
  & Eq. Odds & 11.62\% & .5513 &  N/A & N/A  & 28.46\% & .0962 & 26.65\% & {\bf .0769} \\
  & Disp. Imp.& 14.83\% & .3960 &  N/A & N/A  & 28.26\% & .0557 & 30.26\% & {\bf .0073} \\
\hline
German Stat.  & Dem. Par. & 24.85\% & .0766 & 24.85\%& .0346 & 25.45\%  &  .0410 & 25.15\% &  {\bf .0137} \\
  &Eq. Opp. & 24.85\% & .1120 & 24.54\% & .0922 & 27.27\%  &  .0757 & 25.45\% &  {\bf .0662} \\
  & Eq. Odds & 24.85\% & .1120 &  N/A & N/A  & 34.24\%  &  .1318 & 25.45\% & {\bf .1099} \\
  & Disp. Imp. & 24.85\% & .0608 &  N/A & N/A  & 27.57\%  &  .0468 & 25.15\% &  {\bf .0156}\\
  \hline
  Adult  & Dem. Par. & 14.15\% & .1173 & 16.60\% & .0129 & 20.47\% & .0198 & 16.51\% & {\bf .0037} \\
  & Eq. Opp. & 14.15\% & .1195 & 14.43\% & .0170 & 19.67\% & .0374 & 14.46\% & {\bf .0092} \\
  & Eq. Odds & 14.15\% & .1195 &  N/A & N/A & 19.04\% & {\bf .0160} & 14.58\% & .0221 \\
  & Disp. Imp. & 14.19\% & .1108 &  N/A & N/A & 20.48\% & {\bf .0199} & 17.37\% & .0334 \\
\hline
\end{tabular}
\end{center}
\vspace{-0.2in}
\end{table*}

%We verify our approach empirically across a number of datasets and against some baselines.

\subsection{Datasets}

{\bf Bank Marketing} \citep{lichman2013uci} ($45211$ examples). The data is based on a direct marketing campaign of a banking institution. The task is to predict whether someone will subscribe to a bank product. We use age as a protected attribute: $5$ protected groups are determined based on uniform age quantiles.

{\bf Communities and Crime} \citep{lichman2013uci} ($1994$ examples). Each datapoint represents a community and the task is to predict whether a community has high (above the $70$-th percentile) crime rate. We pre-process the data consistent with previous works, e.g. \citet{cotter2018training} and form the protected group based on race in the same way as done in \citet{cotter2018training}. 
We use four race features as real-valued protected attributes corresponding to percentage of White, Black, Asian and Hispanic. We threshold each 
at the median to form $8$ protected groups. 

{\bf ProPublica’s COMPAS} \cite{propublica2018} Recidivism data ($7918$ examples).
The task is to predict recidivism based on 
criminal history, jail and prison time, demographics,
and risk scores. The protected groups are two race-based (Black, White) and two gender-based (Male, Female).

{\bf German Statlog Credit Data} \citep{lichman2013uci} ($1000$ examples). The task is to predict whether an individual is a good or bad credit risk given attributes related to the individual's financial situation. We form two protected groups based on an age cutoff of $30$ years.

{\bf Adult} \citep{lichman2013uci} ($48842$
examples). The task is to predict whether the person's income is more than $50$k per year.
We use $4$ protected groups based on gender (Male and Female) and race (Black and White). We follow an identical procedure to \citet{zafar2015fairness,goh2016satisfying} to pre-process the dataset. % and use a linear model. 

\subsection{Baselines}

For all of the methods except for the Lagrangian, we train using Scikit-Learn's Logistic Regression \cite{pedregosa2011scikit} with default hyperparameter settings. We test our method against the unconstrained baseline, post-processing calibration, and the Lagrangian approach with hinge relaxation of the constraints. For all of the methods, we fix the hyperparameters across all experiments. For implementation details and hyperparameter settings, see the Appendix.
 %We use logistic regression as the classifier.

\subsection{Fairness Notions}

For each dataset and method, we evaluate our procedures with respect to demographic parity, equal opportunity, equalized odds, and disparate impact. %For all of the fairness notions except disparate impact, the classifier has access to the protected attributes. In disparate impact, it is assumed that we have the protected attributes during training but not evaluation. 
As discussed earlier, the post-processing calibration method cannot be readily applied to disparate impact or equalized odds (without added complexity and randomized classifiers) so we do not show these results.

\subsection{Results}

\begin{figure}[t]
  \begin{center}
    \includegraphics[width=0.4\textwidth]{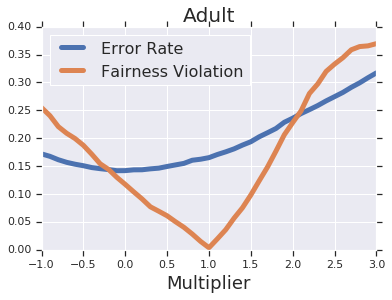}
    \vspace{-0.2in}
    \end{center}
  \caption{{\bf Results as $\lambda$ changes}: We show test error and fairness violations for demographic parity on Adult as the weightings change. We take the optimal $\lambda = \lambda^*$ found by Algorithm~\ref{alg:demparity}. Then for each value $x$ on the $x$-axis, we train a classifier with data weights based on the setting $\lambda = x\cdot \lambda^*$ and plot the error and violations. We see that indeed, when $x = 1$, we train based on the $\lambda$ found by Algorithm~\ref{alg:demparity} and thus get the lowest fairness violation. On the other hand, $x=0$ corresponds to training on the unweighted dataset and gives us the lowest prediction error. Analogous charts for the rest of the datasets can be found in the Appendix. }
	\label{fig:charts_full}
    \vspace{-0.2in}
\end{figure}

We present the results in Table~\ref{tab:experiments}.  We see that our method consistently leads to more fair classifiers, often yielding a classifier with the lowest test violation out of all methods.  
We also include test error rates in the results.  Although the primary objective of these algorithms is to yield a fair classifier, we find that our method is able to find reasonable trade-offs between fairness and accuracy. Our method often provides either better or comparative predictive error than the other fair classification methods (see Figure~\ref{fig:charts_full} for more insight into the trade-offs found by our algorithm).

The results in Table~\ref{tab:experiments} also highlight the disadvantages of existing methods for training fair classifiers.  Although the calibration method is an improvement over an unconstrained model, it is often unable to find a classifier with lowest bias.  %This aligns with previous work~\citep{woodworth2017learning}, which showed that despite the fact that post-processing calibration strives to achieve fairness without regards to accuracy, there exist certain cases where post-processing is provably suboptimal.

We also find the results of the Lagrangian approach to not consistently provide fair classifiers.  
As noted in previous work~\cite{cotter2018optimization,goh2016satisfying}, constrained optimization can be inherently unstable or requires a certain amount of slack in the objective as the constraints are typically relaxed to make gradient-based training possible and for feasibility purposes. Moreover, due to the added complexity of this method, it can overfit and have poor fairness generalization as noted in~\cite{cotter2018training}. 
Accordingly, we find that the Lagrangian method often yields poor trade-offs in fairness and accuracy, at times yielding classifiers with both worse accuracy and more bias. 

\vspace{-0.1in}
\section{MNIST with Label Bias}
\vspace{-0.05in}
We now investigate the practicality of our method on a larger dataset.
We take the MNIST dataset under the standard train/test split and then randomly select $20\%$ of the training data points and change their label to $2$, yielding a biased set of labels. On such a dataset, our method  should be able to find appropriate weights so that training on the weighted dataset roughly corresponds to training on the true labels.
To this end, we train a classifier with a demographic-parity-like constraint on the predictions of digit $2$; i.e., we encourage a classifier to predict the digit $2$ at a rate of $10\%$, the rate appearing in the true labels. 
We compare to the same baseline methods as before. 
See the Appendix for further experimental details.
\begin{table}
\vspace{-0.2in}
\setlength{\tabcolsep}{5pt}
\begin{center}
  \caption{{\bf MNIST with Label Bias}}
  \label{tab:mnist}
  \begin{tabular}{|l|r|}
 \hline
Method & Test Accuracy \\
\hline
Trained on True Labels & 97.85\% \\
\hline
Unconstrained &  88.18\% \\
Calibration & 89.79\% \\
Lagrangian & 94.05\% \\
Our Method & {\bf 96.16\%} \\
\hline
\end{tabular}
\end{center}
\vspace{-0.3in}
\end{table}
We present the results in Table~\ref{tab:mnist}.
We report test set accuracy computed with respect to the true labels.
We find that our method is the only one that is able to approach the accuracy of a classifier trained with respect to the true labels.
Compared to the Lagrangian approach or calibration, our method is able to improve error rate by over half.  Even compared to the next best method (the Lagrangian), our proposed technique improves error rate by roughly 30\%.
These results give further evidence of the ability of our method to effectively train on the underlying, true labels despite only observing biased labels.
\vspace{-0.1in}
\section{Conclusion}
\vspace{-0.05in}
%Datasets contain biases that may become learned in the model.
We presented a new framework to model how bias can arise in a dataset, assuming that there exists an unbiased ground truth. Our method for correcting for this bias is based on re-weighting the training examples. Given the appropriate weights, we showed with finite-sample guarantees that the learned classifier will be approximately unbiased. We gave practical procedures which approximate these weights and showed that the resulting algorithm leads to fair classifiers in a variety of settings.

\clearpage

{\large \bf Acknowledgements}

We thank Maya Gupta, Andrew Cotter and Harikrishna Narasimhan for many insightful discussions and suggestions as well as Corinna Cortes for helpful comments.
{
\bibliography{ref}
\bibliographystyle{plainnat}
}

\clearpage
{
\appendix
\onecolumn
{\Large \bf Appendix}
\section{Sampling Technique}

We present an alternative to the weighting technique.
 For the {\em sampling} technique, we note that the distribution $P(Y=y) \propto \ybiased(y|x) \cdot \exp\left\{ \sum_{k=1}^K \lambda_k c_k(x,y) \right\}$ corresponds to the conditional distribution, 
\begin{equation*}
P(A=y~\text{and}~B=y | A = B),    
\end{equation*}
where $A$ is a random variable sampled from $\ybiased(y|x)$ and $B$ is a random variable sampled from the distribution $P(B=y)\propto\exp\left\{ \sum_{k=1}^K \lambda_k c_k(x,y) \right\}$.
Therefore, in our training procedure for $\model$, given a data point $(x,y)\sim\D$, where $y$ is sampled according to $\ybiased$ (i.e., $A$), we sample a value $y^\prime$ from the random variable $B$, and train $\model$ on  $(x,y)$ if and only if $y = y^\prime$.  This procedure corresponds to training $\model$ on data points $(x,y)$ with $y$ sampled according to the true, unbiased label function $\ytrue(x)$.

The sampling technique ignores or skips data points when $A\ne B$ (i.e., when the sample from $P(B=y)$ does not match the observed label).  In cases where the cardinality of the labels is large, this technique may ignore a large number of examples, hampering training.  For this reason, the {\em weighting} technique may be more practical.

\section{Algorithms for Other Notions of Fairness}

{\bf Equal Opportunity}: Algorithm~\ref{alg:demparity} can be directly used by replacing the demographic parity constraints with equal opportunity constraints. Recall that in equal opportunity, the goal is for the positive prediction rates on the positive examples of the protected group $\G$ to match that of the overall. If the positive prediction rate for positive examples $\G$ is less than that of the overall, then Algorithm~\ref{alg:demparity} will up-weight the examples of $\G$ which are positively labeled. This encourages the classifier to be more accurate on the positively labeled examples of $\G$, which in other words means that it will encourage the classifier to increase its positive prediction rate on these examples, thus leading to a classifier satisfying  equal opportunity. In this way, the same intuitions supporting the application of Algorithm~\ref{alg:demparity} to demographic parity or disparate impact also support its application to equal opportunity.  We note that in practice, we do not have access to the true labels function, so we approximate the constraint violation $\E_{x\sim\D_{[n]}}[\langle h(x), c_k(x) \rangle]$ using the observed labels as $\E_{(x,y)\sim\D_{[n]}}[h(x)\cdot c_k(x, y)]$.

\begin{algorithm}[H]
   \caption{Training a fair classifier for Equalized Odds.}
   \label{alg:equalodds}
\begin{algorithmic}
   \STATE {\bf Inputs}: Learning rate $\eta$, number of loops $T$, training data $\D_{[n]} = \{(x_i, y_i)\}_{i=1}^N$, classification procedure $H$. True positive rate constraints $c_1^{TP},...,c_K^{TP}$ and false positive rate constraints $c_1^{FP},...,c_K^{FP}$ respectfully corresponding to protected groups $\G_1,...,\G_K$.
   \STATE Initialize $\lambda_1^{TP},...,\lambda_K^{TP}, \lambda_1^{FP},...,\lambda_K^{FP}$ to $0$ and $w_1 = w_2 = \cdots = w_n = 1$.
    Let $h := H(\D_{[n]}, \{w_i\}_{i=1}^n)$
   \FOR{$t = 1,...,T$}
   \STATE Let $\Delta_k^{A} := \E_{x\sim\D_{[n]}}[\langle h(x), c_k^{A}(x) \rangle]$ for $k \in [K]$ and $A \in \{TP, FP\}$.
   \STATE Update $\lambda_k^{A} = \lambda_k^{A} - \eta \cdot \Delta_k^{A}$ for $k  \in [K]$ and $A \in \{TP, FP\}$.
   \STATE  $\widetilde{w_i}^{T} := \exp\left( \sum_{k=1}^K \lambda_k^{TP} \cdot \1[x \in \G_k] \right)$ for $i \in [n]$.
    \STATE  $\widetilde{w_i}^{F} := \exp\left(-\sum_{k=1}^K \lambda_k^{FP} \cdot \1[x \in \G_k] \right)$ for $i \in [n]$.
   \STATE Let $w_i = \widetilde{w_i}^{T} / (1 + \widetilde{w_i}^{T})$ if $y_i = 1$, otherwise $w_i = \widetilde{w_i}^{F} / (1 + \widetilde{w_i}^{F})$ for $i \in [n]$
   \STATE Update $h = H(\D_{[n]}, \{w_i\}_{i=1}^n)$
   \ENDFOR 
   \STATE {\bf Return} $h$
\end{algorithmic}
\end{algorithm}

{\bf Equalized Odds}: Recall that equalized odds requires that the conditions for equal opportunity (regarding the true positive rate) to be satisfied and in addition, the false positive rates for each protected group match the false positive rate of the overall. Thus, as before, for each true positive rate constraint, we see that if the examples of $\G$ have a lower true positive rate than the overall, then up-weighting positively labeled examples in $\G$ will encourage the classifier to increase its accuracy on the positively labeled examples of $\G$, thus increasing the true positive rate on $\G$. Likewise, if the examples of $\G$ have a higher false positive rate than the overall, then up-weighting the negatively labeled examples of $\G$ will encourage the classifier to be more accurate on the negatively labeled examples of $\G$, thus decreasing the false positive rate on $\G$. This forms the intuition behind Algorithm~\ref{alg:equalodds}. We again approximate the constraint violation $\E_{x\sim\D_{[n]}}[\langle h(x), c_k^A(x) \rangle]$ using the observed labels as $\E_{(x,y)\sim\D_{[n]}}[h(x)\cdot c_k^A(x, y)]$ for $A\in\{TP,FP\}$.

{\bf More general constraints}: It is clear that our strategy can be further extended to any constraint that can be expressed as a function of the true positive rate and false positive rate over any subsets (i.e. protected groups) of the data. Examples that arise in practice include equal accuracy constraints, where the accuracy of certain subsets of the data must be approximately the same in order to not disadvantage certain groups, and high confidence samples, where there are a number of samples which the classifier ought to predict correctly and thus appropriate weighting can enforce that the classifier achieves high accuracy on these examples.

\section{Proof of Proposition~\ref{prop:closedform}}
\begin{proof}[Proof of Proposition~\ref{prop:closedform}]
The constrained optimization problem stated in Assumption~\ref{ass:bias} is a convex optimization with linear constraints.  We may use the Lagrangian method to transform it into the following min-max problem:
\begin{equation*}
\min_{\ydummy:\X\to[0, 1]} \max_{\lambda_1,\dots,\lambda_K\in\R} J(\ydummy,\lambda_1,\dots,\lambda_K),
\end{equation*}
where $J(\ydummy,\lambda_1,\dots,\lambda_K)$ is define  as
\begin{equation*}
 \E_{x\sim \gP}\left[\dkl(\ydummy(x) || \ytrue(x)) + \sum_{k=1}^K\lambda_k (\langle \ydummy(x), c_k(x) \rangle -\epsilon_k) \right].
\end{equation*}
Note that the KL-divergence may be written as an inner product:
\begin{equation*}
    \dkl(\ydummy(x) || \ytrue(x)) = \langle \ydummy(x), \log \ydummy(x) - \log \ytrue(x) \rangle.
\end{equation*}
Therefore, we have
\begin{align*}
 J(\ydummy,\lambda_1,\dots,\lambda_K) = \E_{x\sim \gP}\Big[&\Big\langle \ydummy(x), \log \ydummy(x) - \log \ytrue(x)+ \sum_{k=1}^K\lambda_k c_k(x)\Big\rangle  - \sum_{k=1}^K\lambda_k \epsilon_k \Big].
\end{align*}
In terms of $\ydummy$, this is a classic convex optimization problem~\citep{botev2011generalized}.  Its optimum is a Boltzmann distribution of the following form:
\begin{equation*}
\ydummy^*(y|x) \propto \exp\left\{ \log \ytrue(y|x) - \sum_{k=1}^K\lambda_k c_k(x,y) \right\}.
\end{equation*}
The desired claim immediately follows.
\end{proof}

\section{Proof of Theorem~\ref{theo:rates}}

\begin{proof}[Proof of Theorem~\ref{theo:rates}]
The corresponding weight for each sample $(x, y)$ is $w(x)$ if $y=1$ and $1 - w(x)$ if $y=0$, where
\begin{align*}
    w(x) := \sigma\left(\sum_{i=1}^K \lambda_k \cdot c_k(x, 1) \right),
\end{align*}
$\sigma(t) := \exp(t) / (1 + \exp(t))$.
Then, we have by Proposition~\ref{prop:closedform} that
\begin{align}\label{eqn:mult_condition}
    \ytrue(x) = \frac{w(x) \cdot \ybiased(x) }{w(x) \cdot \ybiased(x) + (1 - w(x)) (1 - \ybiased(x))}.
\end{align}

Let us denote $w_i$ the weight for example $(x_i, y_i)$.
Suppose that $h^*$ is the optimal learner on the re-weighted objective. That is,
\begin{align*}
    h^* \in \argmin_{h\in\mathcal{H}} J(h)
\end{align*}
where  $J(h) := \frac{1}{n} \sum_{i=1}^n w_i\cdot   \ell(h(x_i), y_i)$ and  $\ell(\hat{y}, y) = (\hat{y} - y)^2$.
Let us partition $\mathcal{X}$ into a grid of $D$-dimensional hypercubes with diameter $R$, and let this collection be $\Omega_R$. For each $B \in \Omega_R$, let us denote the center of $B$ as $B_c$. Define
%\begin{align*}
%    J_R(h) &:= \frac{1}{n} \sum_{B \in \Omega_R} \Big((1 - \ybiased(B_c)) (1 - %w(B_c)) h(B_c)^2 +  \ybiased(B_c) w(B_c) (1 - h(B_c))^2\Big)  \cdot |B \cap X_{[n]}|,
%\end{align*}
%for each $B \in \Omega_R$:
\begin{align*}
    %J_B(h) &:= \frac{1}{|B \cap X_{[n]}|} \sum_{x \in B} w_i\cdot \ell(h(x_i), y_i) \\
    J_{R}(h) &:= \frac{1}{n} \sum_{B \in \Omega_B} \Big( (1 - \ybiased(B_c)) (1 - w(B_c)) h(B_c)^2 +  \ybiased(B_c) w(B_c) (1 - h(B_c))^2 \Big)\cdot |B \cap X_{[n]}| ,
\end{align*}
where $X_{[n]} := \{x_1,..,x_n\}$.

We now show that $|J_{R}(h) - J(h)| \le E_{R, n} := 6LR + \frac{ C\cdot \log(2/\delta)}{n}\sum_{B \in\Omega_R} \sqrt{ |B\cap X_{[n]}|}$. We have
\begin{align*}
    J(h) &= \frac{1}{n} \sum_{i=1}^n w_i\cdot   \ell(h(x_i), y_i) = \frac{1}{n} \sum_{B \in \Omega_R} \sum_{x_i \in B} w_i \cdot  \ell(h(x_i), y_i) \\
    &=  \frac{1}{n} \sum_{B \in \Omega_R} \left(\sum_{\substack{x_i \in B \\ y_i = 0}} w_i \cdot h(x_i)^2 +  \sum_{\substack{x_i \in B\\ y_i = 1}} w_i\cdot (1 -  h(x_i))^2 \right)\\
    &\le \frac{1}{n} \sum_{B \in \Omega_R}   \left(\sum_{\substack{x_i \in B \\ y_i = 0}} w_i \cdot h(B_c)^2 +  \sum_{\substack{x_i \in B\\ y_i = 1}} w_i\cdot (1 - h(B_c))^2 \right) + 4L^2 R^2 +  4 LR\\
    &\le\frac{1}{n} \sum_{B \in \Omega_R} \Big( (1 - \ybiased(B_c))(1  - w(B_c)) h(B_c)^2  + \ybiased(B_c) w(B_c) (1 - h(B_c))^2\Big)\cdot |B \cap X_{[n]}| \\
 &\hspace{0.5cm}+ 5LR +4L^2R^2   + \frac{ C\cdot \log(2/\delta)}{n}\sum_{B \in\Omega_R} \sqrt{ |B\cap X_{[n]}|} \\
 &= J_R(h) +  5LR +4L^2R^2   + \frac{ C\cdot \log(2/\delta)}{n}\sum_{B \in\Omega_R} \sqrt{ |B\cap X_{[n]}|}  \\
 &\le J_{R} (h) +  E_{R, n}.
\end{align*}
where the first inequality holds by smoothness of $h$; the second inequality holds because the value of $w(x)$ for each $x \in B$ will be the same assuming that $R$ is chosen sufficiently small to not allow examples from different protected attributes to be in the same $B \in \Omega_R$ and then applying Bernstein's concentration inequality so that this holds with probability at least $1 - \delta$  for some constant $C > 0$; finally the last inequality holds for $R$ sufficiently small. 
Similarly, we can show that $ J(h) \ge J_{R}(h) - E_{R, n}$, as desired.

It is clear that $\ytrue \in \argmin_{h\in\mathcal{H}} J_{R}(h)$.

We now bound the amount $h^*$ can deviate from $\ytrue$ at $B_C$ on average. Let $h^*(B_c) = \ytrue(B_c) + \epsilon_B$.
Then, we have
\begin{align*}
 2E_{R, n} &\ge  J_R(h^*) - J_R(\ytrue)
\end{align*}
 because otherwise,
\begin{align*}
     J(h^*)  \ge J_{R}(h^*) - E_{R, n} > J_R(\ytrue) + E_{R, n} \ge J(\ytrue),
\end{align*}
contradicting the fact that $h^*$ minimizes $J$. 

We thus have
\begin{align*}
     2E_{R, n} &\ge  J_R(h^*) - J_R(\ytrue) \\
     &= \frac{1}{n} \sum_{B \in \Omega_B} \Big( (1 - \ybiased(B_c))(1  - w(B_c)) (h^*(B_c)^2 - \ytrue(B_c)^2) \\
     &\hspace{1.5cm} + \ybiased(B_c) w(B_c) ((1 - h^*(B_c))^2 - (1 -\ytrue(B_c))^2) \Big)\cdot |B \cap X_{[n]}| \\
     &= \frac{1}{n} \sum_{B \in \Omega_B} \epsilon_B^2 \Big((1 - \ybiased(B_c))(1  - w(B_c))  + \ybiased(B_c) w(B_c)\Big) \cdot |B \cap X_{[n]}| \\
     &\ge \frac{\exp(-K\Lambda)}{1 + \exp(-K\Lambda)} \cdot \frac{1}{n} \sum_{B \in \Omega_B} \epsilon_B^2\cdot |B \cap X_{[n]}|,
\end{align*}
where the last inequality follows by lower bounding $\min\{w(B_C), 1  - w(B_C) \}$ in terms of $\Lambda$.

Thus,
\begin{align*}
    \frac{1}{n} \sum_{B \in \Omega_B} \epsilon_B^2\cdot |B \cap X_{[n]}|
    \le \frac{2 (1 + \exp(-K\Lambda))}{\exp(-K\Lambda)} \cdot E_{R, n}.
\end{align*}
By the smoothness of $h^*$ and $\ytrue$, it follows that
\begin{align*}
\frac{1}{n}\sum_{x \in X_{[n]}} (h^*(x) - \ytrue(x))^2 \le \frac{4 (1 + \exp(-K\Lambda))}{\exp(-K\Lambda)} \cdot E_{R, n}.
\end{align*}

All that remains is to bound this quantity.
By Lemma~\ref{lemma:rms}, there exists constant $C_1$ such that 
$\frac{1}{n} \sum_{B\in \Omega_R} \sqrt{ |B\cap X_{[n]}|} \le C_1 \sqrt{\frac{1}{n R^D}}$.
We now see that for $n$ sufficiently large depending on the data distribution (with the understanding that $R\rightarrow 0$ and thus it suffices to not consider dominated terms), there exists constant $C''$ such that the above is bounded by
\begin{align*}
     C''\left(R +  \log(2/\delta) \sqrt{\frac{1}{n  R^D}}\right).
\end{align*}
Now, choosing $R = n^{-1/(2+D)}\cdot \log(2/\delta)^{2/(2+D)}$, we obtain the desired result.
\end{proof}

\begin{lemma}\label{lemma:rms}
There exists constant $C_1$ such that 
\begin{align*}
    \frac{1}{n} \sum_{B\in \Omega_R} \sqrt{ |B\cap X_{[n]}|} \le C_1 \sqrt{\frac{1}{n R^D}}.
\end{align*}
\end{lemma}
\begin{proof}
We have by Root-Mean-Square Arithmetic-Mean Inequality that
\begin{align*}
&\frac{1}{n} \sum_{B\in \Omega_R} \sqrt{ |B\cap X_{[n]}|} = 
\frac{|\Omega_R|}{n} \frac{1}{|\Omega_R|}\sum_{B\in \Omega_R} \sqrt{ |B\cap X_{[n]}|} \le \frac{|\Omega_R|}{n} \sqrt{\frac{1}{|\Omega_R|}\sum_{B\in \Omega_R} |B\cap X_{[n]}|} = \sqrt{\frac{|\Omega_R|}{n}}  \le C_1 \sqrt{\frac{1}{n R^D}},
\end{align*}
for some $C_1 > 0$
where the inequality follows because the support $\mathcal{X}$ is compact and thus the size of the hypercube partition will grow at a rate of $1/R^D$.
\end{proof}

\section{Proof of Theorem~\ref{theo:rates_manifold}}

The following gives a lower bound on the volume of a ball in $\mathcal{X}$. The result follows from Lemma 5.3 of \cite{niyogi2008finding} and has been used before e.g. \citep{balakrishnan2013cluster,jiang2017density} so we omit the proof. 

\begin{lemma}\label{lemma:manifold_ball} Let $\mathcal{X}$ be a compact $d$-dimensional Riemannian submanifold of $\mathbb{R}^D$ with finite volume and finite condition number $1/\tau$.
Suppose that $0 < r < 1/\tau$. Then, 
\begin{align*}
    \text{Vol}_d(B(x, r) \cap \mathcal{X}) \ge v_dr^d(1 - \tau^2r^2),
\end{align*}
where $v_d$ denotes the volume of a unit ball in $\mathbb{R}^d$ and $\text{Vol}_d$ is the volume w.r.t. the uniform measure on $\mathcal{X}$.
\end{lemma}

We now give an analogue to Lemma~\ref{lemma:rms} but for the manifold setting.

\begin{lemma}\label{lemma:rms_manifold} Suppose the conditions of Theorem~\ref{theo:rates_manifold}.
There exists constant $C_1$ such that 
\begin{align*}
    \frac{1}{n} \sum_{B\in \Omega_R} \sqrt{ |B\cap X_{[n]}|} \le C_1 \sqrt{\frac{1}{n R^d}}.
\end{align*}
\end{lemma}
\begin{proof}

We have by Lemma~\ref{lemma:manifold_ball} that there exists $C_2 > 0$ such that for each $B \in \Omega_R$, and $R$ sufficiently small,
$\text{Vol}_d(B \cap \mathcal{X}) \ge C_2\cdot R^d$.
Thus,
\begin{align*}
|\Omega_R| \le \frac{\text{Vol}_d(\mathcal{X})}{ C_2\cdot R^d}.
\end{align*}
We then continue as in Lemma~\ref{lemma:rms} proof and have by Root-Mean-Square Arithmetic-Mean Inequality that
\begin{align*}
&\frac{1}{n} \sum_{B\in \Omega_R} \sqrt{ |B\cap X_{[n]}|} = 
\frac{|\Omega_R|}{n} \frac{1}{|\Omega_R|}\sum_{B\in \Omega_R} \sqrt{ |B\cap X_{[n]}|} \le \frac{|\Omega_R|}{n} \sqrt{\frac{1}{|\Omega_R|}\sum_{B\in \Omega_R} |B\cap X_{[n]}|} = \sqrt{\frac{|\Omega_R|}{n}}  \le C_1 \sqrt{\frac{1}{n R^d}},
\end{align*}
for some $C_1 > 0$, as desired.
\end{proof}

\begin{proof}[Proof of Theorem~\ref{theo:rates_manifold}]
The proof is the same as that of Theorem~\ref{theo:rates} except we can now replace the usage of Lemma~\ref{lemma:rms} with Lemma~\ref{lemma:rms_manifold} to obtain rates in terms of $d$ instead of $D$.
\end{proof}

\section{Further Experimental Details}
\subsection{Baselines}

{\bf Unconstrained (Unc.)}: This method applies logistic regression on the dataset without any consideration for fairness.

{\bf Post-calibration (Cal.)}: This method~\cite{hardt2016equality} first trains without consideration for fairness, and then determines appropriate thresholds for the protected groups such that fairness is satisfied in training. In the case of overlapping  groups, we treat each intersection as their own group. 

{\bf Lagrangian approach (Lagr.)} This method \citep{eban2017scalable} proceeds by jointly training the Lagrangian in both model parameters and Lagrange multipliers and uses a hinge approximation of the constraints to make the Lagrangian differentiable in its input. We fix the model learning rate and use ADAM optimizer with learning rate $0.01$ and train for $100$  passes through the dataset. We also had to select a slack for the constraints as without slack, we often converged to degenerate solutions. We chose the smallest slack in increments of $5\%$ until the procedure returned a non-degenerate solution.  

{\bf Our method}: We use Algorithm~\ref{alg:demparity} for demographic parity, disparate impact, and equal opportunity and Algorithm~\ref{alg:equalodds} for equalized odds. We fix the learning rate $\eta = 1.$ and number of loops $T = 100$ across all of our experiments.

\subsection{MNIST Experiment Details}
We use a three hidden-layer fully-connected neural network with ReLU activations and $1024$ hidden units per layer and train using TensorFlow's ADAM optimizer under default settings for $10000$ iterations with batchsize $50$.

For the calibration technique, we consider the points whose prediction was 2 and then for the lowest softmax probabilities among them, swap the label to the prediction corresponding to the second highest logit so that the prediction rate for 2 is as close to $10\%$ as possible. For the Lagrangian approach, we use the same settings as before, but use the same learning rate as the other procedures for this simulation.
For our method, we adopt the same settings as before.

\section{Additional Experimental Results}

\begin{figure}[H]
  \begin{center}
    \includegraphics[width=0.3\textwidth]{figures/Adult}
    \includegraphics[width=0.3\textwidth]{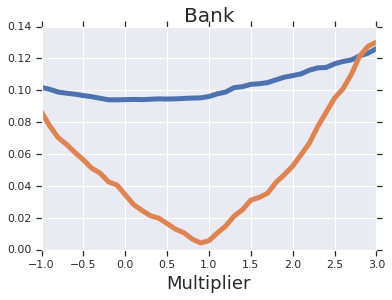} 
    \includegraphics[width=0.3\textwidth]{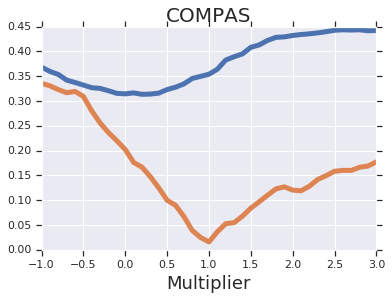}\\
    \includegraphics[width=0.3\textwidth]{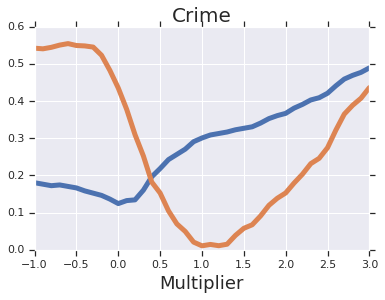}  
    \includegraphics[width=0.3\textwidth]{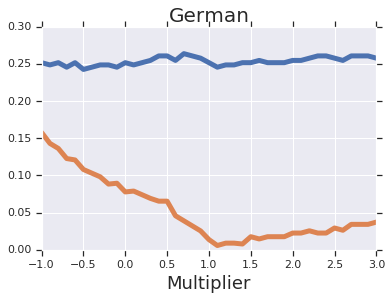}
    \end{center}
  \caption{{\bf Results as $\lambda$ changes}: we show test error and fairness violations for demographic parity as the weightings change. We take the optimal $\lambda = \lambda^*$ found by Algorithm~\ref{alg:demparity}. Then for each value $c$ on the $x$ axis, and we train a classifier with data weights based on the setting $\lambda = c\cdot \lambda^*$ and plot the error and violations. We see that indeed, when $c = 1$, we train based on the $\lambda$ found by Algorithm~\ref{alg:demparity} and thus get the lowest fairness violation. For $c=0$, this corresponds to training on the unweighted dataset and gives us the lowest error. }
	\label{fig:charts}
\end{figure}

}

\end{document}